%% file: main.tex
\documentclass[11pt]{article}

\newif\ifnotes
\notestrue

\input{includes}
\input{macros}

\usepackage{parskip}

\usepackage[
natbib=true,
style=trad-alpha,
maxalphanames=8,
sorting=nyt
]{biblatex}
\numberwithin{equation}{section}
\numberwithin{algorithm}{section}

\allowdisplaybreaks

\begin{document}

\title{\papertitle}
\author{Naren Sarayu Manoj \\ \texttt{nsm@ttic.edu}
\and Nathan Srebro\\ \texttt{nsrebro@ttic.edu}
}
\date{\vspace{-5mm}Toyota Technological Institute at Chicago}

\sloppy
\maketitle
\thispagestyle{empty}
\pagenumbering{arabic}

\begin{abstract}
    \input{abstract.tex}
\end{abstract}

\input{paper}

\end{document}

%% file: includes.tex
\usepackage{hyperref}
\usepackage{fullpage}
\usepackage{amssymb}
\usepackage{amsmath}
\usepackage{mathtools}

\usepackage{amsthm}
\usepackage{amsfonts}
\usepackage{amsmath}
\usepackage{bm}
\usepackage{enumitem}
\usepackage{color}
\usepackage{comment}
\usepackage[capitalize]{cleveref}
\usepackage[dvipsnames]{xcolor}
\usepackage{float}
\usepackage[T1]{fontenc}
\usepackage{todonotes}
\usepackage{asymptote}
\usepackage{mdframed}
\usepackage[most]{tcolorbox}
\usepackage{hyperref}
\usepackage{enumitem}
\usepackage{framed}
\usepackage{mdframed}
\usepackage{scrextend}
\usepackage{bbm}
\usepackage{wrapfig}

\usepackage{mathrsfs}

\newcommand{\natinote}[1]{\ifnotes $\ll$\textsf{\color{magenta} Nati: { #1}}$\gg$ \fi}
\newcommand{\natinotelater}[1]{}

\usepackage[T1]{fontenc}

\definecolor{asparagus}{rgb}{0.53, 0.66, 0.42}
\definecolor{sacramentostategreen}{rgb}{0.0, 0.3, 0.15}
\definecolor{teal}{rgb}{0.0, 0.5, 0.5}
\definecolor{forestgreen}{rgb}{0.13, 0.6, 0.13}

\usepackage{hyperref}
\hypersetup{
    colorlinks=true,
    linkcolor=sacramentostategreen,
    filecolor=sacramentostategreen,
    citecolor=teal,
    urlcolor=teal,
}

\newtheorem{theorem}{Theorem}[section]

\newtheorem*{definition*}{Definition}
\newtheorem{lemma}[theorem]{Lemma}
\newtheorem{corollary}[theorem]{Corollary}

\newcommand{\E}{\mathbb{E}}

\newcommand{\cD}{\mathcal{D}}

\def\to{\rightarrow}
\def\eps{\varepsilon}

\def\eps{\varepsilon}

\renewcommand{\bar}{\overline}

\usepackage{algorithm}
\usepackage{algorithmic}

\newcommand{\inparen}[1]{\left(#1\right)}             
\newcommand{\inbraces}[1]{\left\{#1\right\}}           
\newcommand{\insquare}[1]{\left[#1\right]}             

\newcommand{\B}{\inbraces{0,1}}

\newcommand{\indicator}[1]{\mathbbm{1}\inbraces{#1}}

\newcommand{\abs}[1]{\ensuremath{\left\lvert #1 \right\rvert}}

\usepackage{nicefrac}

\let\nfrac=\nicefrac

\renewcommand{\Pr}{\mathsf{Pr}}

\newcommand{\exv}[1]{\E\insquare{#1}}
\newcommand{\exvv}[2]{\underset{#1}{\E}\insquare{#2}}
\newcommand{\expv}[1]{\mathsf{exp}\inparen{#1}}
\newcommand{\prv}[1]{\Pr\insquare{#1}}
\newcommand{\prvv}[2]{\underset{#1}{\Pr}\insquare{#2}}
\newcommand{\logv}[1]{\log\inparen{#1}}

\makeatletter
\newcommand{\customlabel}[2]{%
   \protected@write \@auxout {}{\string \newlabel {#1}{{#2}{\thepage}{#2}{#1}{}} }%
   \hypertarget{#1}{#2}
}
\makeatother

\newcounter{casenum}

\newcounter{subcasenum}

\newcounter{casenump}

\newcommand{\casep}[2]{
    \ifthenelse{\equal{\value{casenump}}{0}}{
    \vskip.5\baselineskip\par\noindent
    }{}
    {\it Case \arabic{casenump}:} {\it #1}
    \vskip0.1\baselineskip
    \begin{addmargin}[1.5em]{1em}
    #2
    \end{addmargin}
    \addtocounter{casenump}{1}
}

\newcounter{subcasenump}

\usepackage{cleveref}

\usepackage{ulem}

\definecolor{ForestGreen}{RGB}{34, 139, 34}

\newcommand{\removed}[1]{}

%% file: macros.tex
\usepackage{listings}
\usepackage{gnuplottex}
\usepackage{svg}

\newcommand{\jmlrBlackBox}{\rule{1.5ex}{1.5ex}}

\newcommand{\jmlrQED}{\hfill\jmlrBlackBox\par\bigskip}

\DeclarePairedDelimiterX{\infdivx}[2]{(}{)}{%
  #1\;\delimsize\|\;#2%
}
\newcommand{\infdiv}{D_{\mathsf{KL}}\infdivx}

\newcommand{\mdl}{\mathsf{MDL}}
\newcommand{\trainset}{S}

\newcommand{\prog}{\mathsf{h}}

\newcommand{\progw}{\mathsf{u}}
\newcommand{\ProgW}{\mathsf{U}}

\newcommand{\nsamples}{m}
\newcommand{\xlength}{b}
\newcommand{\seed}{\mathsf{seed}}
\newcommand{\seedlength}{r}
\newcommand{\htable}{\mathsf{hash}}
\newcommand{\hstar}{\mathsf{h}^{\star}}

\newcommand{\ber}[1]{\mathsf{Ber}\inparen{#1}}
\newcommand{\kld}[2]{\infdiv*{#1}{#2}}
\newcommand{\sign}[1]{\mathsf{Sign}{#1}}
\newcommand{\Ber}{\mathsf{Ber}}
\newcommand{\rhs}{R}

\newcommand{\gagg}{\ell_{\mathsf{ag}}}
\newcommand{\gsamp}{\ell_{\mathsf{samp}}}

\newcommand{\minH}{H_{\min}}
\newcommand{\minHx}{\minH(X[:\!\xlength])}
\newcommand{\minHxt}{\minH(X[:\!\tilde\xlength])}
\newcommand{\Lstar}{L^{\star}}
\newcommand{\Ltilde}{\widetilde{L}}

\newcommand{\papertitle}{Interpolation Learning With Minimum Description Length}

%% file: abstract.tex
We prove that the Minimum Description Length learning rule exhibits tempered overfitting. We obtain tempered agnostic finite sample learning guarantees and characterize the asymptotic behavior in the presence of random label noise.

%% file: paper.tex
\input{paper_intro}
\input{paper_results}
\input{paper_interpolation}
\input{paper_generalization}
\input{paper_tightness}

\subsection*{Acknowledgments} We thank David McAllester for discussions on $\mdl$ learning, and Madhur Tulsiani and Roei Tell for discussions on random number generators.  We are especially grateful to Alexander Razborov for pointing out the argument in the lower bound in terms $\xlength$ (end of Section \ref{sec:interpolation}), which led to the proof approach for Theorem \ref{thm:short_hash_table} using enumeration as a method of description.

This research was supported in part by the NSF-Simons Collaboration on the Mathematics of Deep Learning and an NSF-Tripod-supported Institute for Data Economics and Algorithms. NSM was funded in part by a United States NSF Graduate Research Fellowship.

\printbibliography

\newpage

\appendix
\input{appendix}

%% file: paper_intro.tex
\section{Introduction}

We consider the minimum description length learning rule $\mdl(\trainset)$, which returns the predictor with minimal description length (in some universal description or programming language) that fits the training set.  MDL learning is well understood in the realizable setting -- if there exists some $\hstar$ that is perfect on the source distribution, i.e.~with zero population loss $L(\hstar)=0$, then $O(\abs{\hstar}/\eps)$ samples are enough for $\mdl$ to have (expected) population loss at most $\eps$, where $\abs{\hstar}$ is the description length of $\hstar$.  But to handle noisy situations, or compete with a short-description predictor $\prog$ that might not be perfect, the standard wisdom is to follow the Structural Risk Minimization (SRM) principle and balance training error against description length. By minimizing the right combination of training error and description length (perhaps tuned through validation), such an SRM predictor can compete with any predictor $\prog$, and using a training set of size $\nsamples$ has expected error at most
\begin{equation}\label{eq:srm}
    \inf_\prog \left( L(\prog) + O\left( \frac{\abs{\prog}}{\nsamples} + \sqrt{L(\prog)\cdot\frac{\abs{\prog}}{\nsamples}} \right) \right)
\end{equation}
But following recent interest in benign overfitting and interpolation learning of noisy data \cite[and many others]{belkin2018overfittingperfectfitting,belkin2019interpolationoptimality,negrea:in-defense,bartlett2020benignpnas,montanari2020maxmarginasymptotics,hastie2020surprises,muthukumar2021classification,chatterji2020linearnoise}, we ask: \textit{what happens if we insist on interpolating (i.e.~obtaining zero training error) and using the interpolating $\mdl$ rule?}  Does $\mdl$ overfit benignly?  Does it still enjoy the same guarantee \eqref{eq:srm} as SRM?  Is it consistent like SRM, i.e., does it converge to the Bayes optimal predictor (as long as the Bayes optimal predictor has finite description)? Or is overfitting by $\mdl$ catastrophic, possibly yielding worthless predictions?  Or perhaps tempered \cite[as defined in][]{mallinar2022benign} with error worse than the optimally balanced SRM, but still better than random guessing?  If so, can we bound the error of the interpolating MDL compared to the optimally balanced SRM?  How much worse can it be compared to the SRM guarantee \eqref{eq:srm} ?

We show that MDL overfitting is not benign, with asymptotic error that could be worse than SRM.  But we can bound this error away from $0.5$, as a simple fixed function of the Bayes error, depicted in Figure \ref{fig}.  For a random label noise model, we obtain a tight and precise characterization of the asymptotic error.  Furthermore, we obtain an agnostic finite sample guarantee, which holds for any source distribution, without any realizability or specification assumptions, and tells us how well we compete with any competitor hypothesis (not necessarily the Bayes optimal).  This contrasts with much of the existing work on benign overfitting which is distribution-specific, e.g.~making Gaussianity assumptions on the data, and often assuming the model is well specified.

Our analysis essentially follows a uniform convergence approach, and decouples the analyses of the description length of $\mdl$ from that of the generalization error for short programs.  In Section \ref{sec:interpolation} we bound the  minimum description length $\abs{\mdl(S)}$ by proving an upper bound on the program needed to interpolate a noisy training set.  Then in Section \ref{sec:generalization} we bound the expected error of any learning rule returning short programs in terms of the length of the program.  Our learning guarantees, stated in Section \ref{sec:main}, then follow immediately by combining the two.  

Although we use an information-theoretic approach in our generalization proofs, the proofs essentially rely on a uniform guarantee over all short programs.  In particular, they hold for {\em any} interpolation rule, not only $\mdl$, and the connection to $\mdl$ is only by plugging in the program length we can ensure for $\mdl$.  This is similar in spirit to the uniform convergence of interpolator arguments of \citet{koehler2021uniform,wang2022tight}, which separately bound the norm of the min-norm predictor, and then analyze uniform convergence over the appropriate norm ball.  

\paragraph{Notation} We write Bernoulli random variables with parameter $\alpha$ as $\ber{\alpha}$. We use $H(X)$ to denote the entropy of random variable. We also write $H(\alpha)$ to denote the entropy of a $\ber{\alpha}$-random variable.  The Radon-Nikodym derivative between two distributions $p$ and $q$ is denoted $dp/dq$, and one can informally think of $dp(\cdot)$ as the probability density or mass function.  We measure information in bits, and $\log$ is always base $2$.  The operation $\oplus$ denotes the XOR of two bits. For two random variables $A$ and $B$, we write $A \perp B$ to mean that $A$ and $B$ are independent.
  

%% file: paper_results.tex
\begin{figure}[t]
    \centering
    \includegraphics[width=\textwidth]{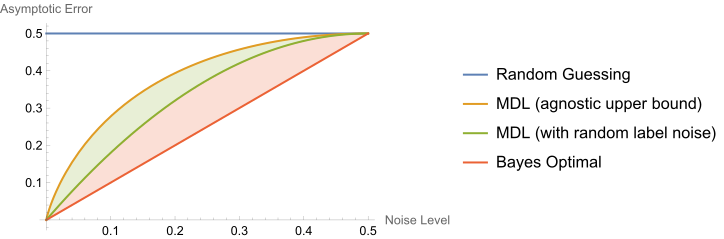}
  \vspace{-9mm}
    \caption{ Behavior of interpolating MDL as a function of the noise level.  Top curve: The function $\gagg(\Lstar)$, which provides an agnostic upper bound on the error of $\mdl$, with a finite sample gurantee that approaches this curve; Contrast with the lower curve: the function $\gsamp(\Lstar)$, which is the exact asymptotic error of $\mdl$ under random label noise.}
    \label{fig}
\end{figure}

\section{Formal Setup and Main Results}\label{sec:main}

We consider learning based on $\nsamples$ i.i.d. samples $S=\{ (x_1,y_1),\ldots,(x_\nsamples,y_\nsamples)\}\sim \cD^\nsamples$ from a source distribution $\cD(X,Y)$ over bit-strings $X$ and binary labels $Y \in \{0,1\}$.  A learning rule is a mapping $A:\trainset\mapsto\prog$ from training sets to predictors.  To formalize the notion of description length with a universal description language, we can think of the predictors $\prog$ as programs in some prefix-unambiguous Turing complete programming language, and we use $\abs{\prog}$ to denote program or description length in bits.  We denote the training error as $L_{\trainset}(\prog)=\tfrac{1}{\nsamples}\indicator{\prog(x_i)\neq y_i}$. We say $\prog$ interpolates $\trainset$ if $L_{\trainset}(\prog)=0$ and that $A$ is an interpolating rule if $L_{\trainset}(A(\trainset))=0$ almost surely.   We denote the population error by $L(\prog)=\prvv{(X,Y)\sim\cD}{\indicator{\prog(x_i)\neq y_i}}$ and we use the same notation for the expected error of a learning rule: $L(A)=\exvv{S\sim\cD^\nsamples}{L_\cD(A(S))}$.

In order to discuss interpolation learning, we must ensure it is always indeed possible to interpolate.  This is the case if we never encounter the same instance $x$ with different labels, i.e.~
\begin{equation}\label{eq:intable}
\prv{ X=X' \wedge Y \neq Y'}=0 \quad\textrm{ for} \quad (X,Y),(X',Y')\sim \textrm{ i.i.d. } \cD.   
\end{equation}
We will thus always assume \eqref{eq:intable}.  This is the case when $Y$ is a deterministic function of $X$.  But we are particularly interested in noisy settings, in which case \eqref{eq:intable} holds if $\cD$ is non-atomic, i.e.~$\prv{X=X'}=0$.  In order to discuss non-atomic distributions over bit-strings, we will allow $X\in\{0,1\}^\mathbb{N}$ to be an infinite\footnote{To capture also finite bit strings $x\in\{0,1\}^*$, we can think of padding $x$ with an infinite number of zeros} sequence of bits (e.g.~the binary digits of a real number).   The programs\footnote{Formally, when discussing programs taking an infinite $x$ as input, we can think of a RAM computer which is allowed random access to bits of $x$, or a Turing Machine given access to $x$ on an infinite tape.} we learn will only be able to access a finite number of bits of $x$, and it will be useful for us to consider prefixes $x[:\!\xlength]$ consisting of the first $\xlength$ bits of $x$.  Although we consider infinite bit sequences, we will need to bound how far we need to read in order to distinguish between instances. We formalize this notion through the following definition (Definition \ref{def:inter_prefix_length}).

\begin{definition*}
\label{def:inter_prefix_length}
The {\bf disambiguation prefix length} $\xlength(\trainset)$ of a sample $\trainset$ is the minimal $\xlength$ such that for all $(x_i, y_i), (x_j, y_j) \in \trainset$, if $x_i[:\!\xlength]=x_j[:\!\xlength]$, then $(x_i,y_i)=(x_j,y_j)$.
The {\bf quenched disambiguation prefix length} $\overline{\xlength}(\nsamples)$ of a distribution $\cD$ for sample size $\nsamples$ is given by
\begin{align*}
    \log \overline{\xlength}(\nsamples) \coloneqq \exvv{\trainset\sim\cD^\nsamples}{\log \xlength(\trainset)} \leq \log \left(\exvv{\trainset\sim\cD^\nsamples}{\xlength(\trainset)}\right)
\end{align*}
\end{definition*}

\noindent With these definitions in hand, we are ready to state our main results.

\begin{theorem}[Agnostic]
\label{thm:agnostic_main}
For any source distribution $\cD$ with quenched disambiguation prefix length $\overline{\xlength}(\nsamples)$, and any sample size $\nsamples$:
\begin{align*}
    \exvv{\trainset\sim\cD^\nsamples}{L(\mdl(\trainset))} \leq \inf_\prog \inparen{ \gagg(L(\prog)) + O\inparen{ \frac{ \abs{\prog} + \log(\nsamples\cdot\overline{\xlength}(\nsamples))}{\nsamples}}}
\end{align*}
where $\gagg(\alpha)\doteq 1-2^{-H(\alpha)}=1-\alpha^\alpha(1-\alpha)^{1-\alpha}$ and $\alpha < \gagg(\alpha)<0.5$ for $0<\alpha<0.5$. 
\end{theorem}

For a ``well specified'' distribution, where the label noise is independent of $X$, we obtain a tighter and more precise guarantee:  

\begin{theorem}[Random Label Noise]
\label{thm:random_label_noise_main} For any source distribution $\cD$ where $Y|X=\hstar(X) \oplus \Ber(\Lstar)$ for some program $\hstar$ and label noise $\Lstar$, and any sample size $\nsamples$:
\begin{align*}
    \abs{ \exvv{\trainset\sim\cD^\nsamples}{L(\mdl(\trainset))} - 
    \gsamp(\Lstar) 
    } 
    \le &O\inparen{\frac{ \abs{\hstar} + \log(\nsamples \cdot \overline{\xlength}(\nsamples))}{\nsamples} 
    + 
    \sqrt{ \Lstar \cdot \frac{ \abs{\hstar} + \log(\nsamples\cdot \overline{\xlength}(\nsamples))}{\nsamples}}}
\end{align*}
where $\gsamp(\Lstar)\doteq 2\Lstar(1-\Lstar)$ and $\Lstar<\gsamp(\Lstar)<\gagg(\Lstar)<0.5$ for $0<\Lstar<0.5$. 
\end{theorem}

\noindent Theorems \ref{thm:agnostic_main} and \ref{thm:random_label_noise_main} follow from plugging in Corollary \ref{corollary:mdl_length} \removed{from Section \ref{sec:interpolation}} into Lemmas \ref{lemma:agnostic_gen_mi} and \ref{lemma:mutual_info_uc}, which we formally establish in Section \ref{sec:generalization}.  The above Theorems hold for any finite number of samples and directly imply guarantees on the asymptotic error of $\mdl$:

\begin{corollary}\label{cor:limit} 
For any source distribution $\cD$ with quenched interpolation length $\overline{\xlength}(\nsamples)\leq2^{o(\nsamples)}$ and such that the Bayes predictor $\prog^{\star}(x) = \sign( P(Y|x) - 0.5)$ is computable, with Bayes error $\Lstar=L(\prog^{\star})<0.5$ then:
$$ \limsup_{\nsamples\rightarrow\infty} \exvv{\trainset\sim\cD^{\nsamples}}{L(\mdl(\trainset))} \leq  \gagg(\Lstar)<0.5  $$
And moreover, if the noise probability is independent of $X$, i.e.~$Y\!\perp\!X|\hstar(X)$, then more precisely:
$$ \lim_{\nsamples\rightarrow\infty} \exvv{\trainset\sim\cD^{\nsamples}}{L(\mdl(\trainset))} = \gsamp(\Lstar)  = 2\Lstar(1-\Lstar)$$
\end{corollary}

\noindent In Figure \ref{fig}, we plot the general upper bound $\gagg$ and the precise error for random label noise $\gsamp$.  We see that even with random label noise, $\mdl$ overfitting is {\em not} benign, and $\mdl$ is not consistent.  Nevertheless, regardless of the noise distribution, the asymptotic error can be non-trivially bounded as a function only of the Bayes error (or rather, the optimal error with a computable predictor), and without any dependence on any other property of the distribution, the predictor, or the noise.

\paragraph{Bounding the quenched interpolation length} 


We can bound $\overline{\xlength}$ in terms of the min-entropy rate $\minHx/\xlength$, where recall the min-entropy is defined as $\minH(X)\doteq-\max_x \log P(X=x)$.  For distributions uniform over $N$ outcomes, this is equal to $\log N$, which is also the Shannon entropy.  But otherwise it can be smaller and captures the ``worst case'' randomness.  
The quenched interpolation length $\overline{\xlength}(\nsamples)$ is roughly the length that ensures no collisions in a sample of size $\nsamples$, i.e., such that $\nsamples^2\cdot \prv{\,X[:\!\xlength]=X'[:\!\xlength]\,} \leq \nsamples^2 2^{-\minH(X[:b])} \ll 1$, and so $\minHxt \approx O(\log \nsamples)$.  If the bits of $X$ are uniform and independent, then the min-entropy rate is $1$, $\minHx=\xlength$ and we have $\overline{\xlength} = O(\log \nsamples)$.  
We can afford a much lower min-entropy rate.  Any constant rate (e.g.~when a small constant fraction of the bits are slightly bounded away from from being fixed conditioned on the previous bits), or arbitrary small polynomial rate $\minHx = \Omega(\xlength^\rho)$, still yields
$\log\overline{\xlength} = O(\log\log \nsamples)$.  Even a logarithmically small min entropy rate, $\minHx = \Omega(\log \xlength)$ 
still ensures $\log\overline{\xlength}(\nsamples)=O(\log \nsamples)$, and so we can ignore the dependence on $\overline{\xlength}$ in our results.  This happens, e.g., when differences between instances become increasingly sparse, with the entropy of the $i$th bit (conditioned on the previous bits) behaving like $1/i$. 
If the min-entropy rate is even lower, down to $\minHx=\log\log\xlength+\omega(1)$, we still have $\overline{\xlength}(\nsamples)\leq 2^{o(\nsamples)}$ and the limits in Corollary \ref{cor:limit} are still valid.

\removed{
\natinote{The roadmap is provided in the intro, where we discuss the proof structure and which section does what.  I did add a sentence on where to find the proof}
\paragraph{Outline} The rest of this paper is organized as follows. In Section \ref{sec:interpolation}, we give a construction of a short program that interpolates any training set consisting of distinct examples (Theorem \ref{thm:short_hash_table}). From this it follows that we can give a program that interpolates any training set whose source distribution has bounded quenched interpolation length (Corollary \ref{corollary:mdl_length}). Then, in Section \ref{sec:generalization}, we prove generalization bounds in terms of the lengths of the program output by any interpolating learning rule, both in the agnostic setting (Lemma \ref{lemma:agnostic_gen_mi}) and in the random label noise setting (Lemma \ref{lemma:mutual_info_uc}). The proofs of Theorems \ref{thm:agnostic_main} and \ref{thm:random_label_noise_main} follow from these as corollaries. Finally, in Section \ref{sec:discussion}, we discuss the tightness and implications of our results.
}

%% file: paper_interpolation.tex
\section{Constructing a Short Interpolating Predictor}
\label{sec:interpolation}

Our goal in this section is to bound the length of a program that interpolates a noisy sample.  In fact, we prove a deterministic worst-case bound on the program length needed to interpolate any given training set.

\vspace{1mm}

\noindent \textbf{Overview and Intuition: How can we construct a short program interpolating a noisy sample?} 

\vspace{1mm}  

\noindent One approach is to memorize the sample $S$, or better yet, encode a good predictor $\prog$ and then memorize all points in the sample that do not agree with $\prog$.  Such an interpolating predictor would generalize as well as $\prog$ (since test examples will mostly not match the memorized examples).  But is it the shortest? It would require storing all instances $x_i$ that do not agree with $\prog$, and thus a description length of $L_S(\prog)\cdot\nsamples\cdot\bar\xlength(S)$.  

The key is that we don't care about memorizing the identities of the instances $x_i$ in the sample. We only need to remember the labels $y_i$, and so we can hope to prevent the description length from scaling linearly with $\xlength$. To encode the information in the labels $y_i$, or rather their disagreement with $\prog(x_i)$, we should need only $m \cdot H(L(\prog))$ bits.  

One approach to doing so is to hash the instances and store the labels (or disagreements) of the hash values.  We could do this if our hash function has no collisions on $\trainset$. The challenge in this approach is to determine how many bits are required to encode a hash function that is collision-free on $\trainset$. Observe that such a function cannot be totally independent of $\trainset$, since any hash function would have collisions on some $\trainset$. Hence, any such hash function requires a description with super-constant length.

We take a more direct approach. We ask how difficult it would be to find and describe a ``hash function'' mapping instances to single bits such that the output values on the sample are exactly what we need them to be.  Consider using a ``random'' binary function $\htable(x)$ where $\htable(x)\sim(\ber{L(\prog)})$.  Such a random function will interpolate with probability roughly $L(\prog)^{L(\prog) \nsamples}(1-L(\prog))^{(1-L(\prog))\nsamples}=2^{-\nsamples H(L(\prog))}$.  If we use a pseudo-random generator with seed length $\gg \nsamples H(L(\prog))$, one of the $\gg 2^{\nsamples H(\alpha)}$ ``random'' functions, corresponding to some specific seed value, should hopefully interpolate. We can then describe this function through its corresponding seed.

But how can we guarantee that some seed would work?  To match the above probability calculation to the output of a pseudo-random generator (PRG), we need a PRG that generates $N$ bits that are $\nsamples$-way independent and marginally $\ber{\alpha}$ using a seed of length $\nsamples H(\alpha)+O(\log m + \log\log N)$ (we need to generate $N=2^{\overline{\xlength}}$ bits, for each possible input $x$).  We are not aware of any explicit PRG allowing this.  Instead, the approach we take is to prove such a PRG must exist (Lemma \ref{lemma:hash_table_exists}) and then describe it as ``the lexicographically first such PRG.''  This is a perfectly valid and precise description that can be encoded as a constant length program.

Notice that unlike the expensive instance memorization approach, the random hash predictor will not generalize as well as $\prog$.  The output $\htable(x)$ will have the same bias $L(\prog)$ on test instances, leading to a test error of $2L(\prog)(1-L(\prog))$ (we make a mistake either if $\prog$ does and we didn't correct it, or if $\prog$ didn't make a mistake but we accidentally corrected it).  In Section \ref{sec:generalization}, we show through Lemma \ref{lemma:mutual_info_uc} that the $\mdl$ predictor indeed behaves this way.

\paragraph{Formal Results} We establish a worst-case (deterministic) bound on the program length needed to memorize any labels (which we can think of as noise), in terms of the the bias of the labels.  We then use this to describe a short program  that interpolates the disagreement vs.~a reference predictor on a random training set.

\removed{
The main issue here is bounding the length of the program required to memorize the noise $z_i=y_i\oplus \prog(x_i)$.  One option is to memorize the entire sample, or at least the noisy points where $z_i=1$.  The number of noisy points is $\nsamples L_S(\prog)$.  The linear scaling with $\nsamples$ is unavoidable, but the problem is that memorizing each point requires memorizing not only its label $y_i$, but also $x_i$, so that our program can compare its input to each one of the memorized points.  Such an approach indeed overfits benignly (random test points won't match with any training point, and so the output on test points will be the same as the ``clean'' program $\prog$), but }

\begin{theorem}
\label{thm:short_hash_table}
Let $\trainset = \inbraces{(x_i, y_i), \text{ for } i \in [\nsamples]}$, where $x_i \in \B^{\xlength}$, $y_i \in \B$, and the $x_i$ are pairwise distinct.  Then, there exists a program $\prog$ of length
\begin{align*}
    \abs{\prog} = \nsamples\cdot H\left(\frac{\sum_i y_i}{\nsamples}\right) + 3\log\nsamples + \log\xlength + O(1)
\end{align*}
such that for all $(x_i, y_i) \in S$, we have $\prog(x_i) = y_i$.
\end{theorem}

For any program $\prog$, we can apply Theorem \ref{thm:short_hash_table} to the ``labels'' $y_i\oplus \prog(x_i)$ to obtain the following Corollary (Corollary \ref{corollary:mdl_length}).

\begin{corollary}
\label{corollary:mdl_length}
For $\trainset \sim \cD^\nsamples$ with quenched interpolation length $\overline{\xlength}$, we have
\begin{align*}
    \exvv{\trainset}{\abs{\mdl(\trainset)}} \le \underset{\text{programs } \prog}{\min} \inbraces{\abs{\prog} + \nsamples \cdot H\!\inparen{L(\prog)} + O\inparen{\log\nsamples + \log\overline{\xlength}(\nsamples)}}
\end{align*}
\end{corollary}

\begin{proof}
For any program $\prog$ and any $S$, let $\widetilde{\prog_S}$ be the short program ensured by Theorem \ref{thm:short_hash_table} for 
$\widetilde{S} = \left\{\ \left(\, x_i[:\!\overline{\xlength}(S)] \,,\, y_i\oplus \prog(x_i)\,\right) \removed{\middle| (x_i,y_i)\in S}\right\}$.  If $\widetilde{S}$ has repeated points, we remove them---Lemma \ref{lemma:removing_duplicates} in the Appendix shows that removing duplicates can only reduce $\nsamples H(\sum y_i/\nsamples)$, and so also the guaranteed length.  The program $\prog_S(x)=\prog(x)\oplus\widetilde{\prog_S}(x[:\!\overline{\xlength}(S)])$ interpolates $S$ and is of length $\abs{\prog}+\abs{\widetilde{\prog_S}}+\log(\overline{\xlength}(S))+O(1)\leq \abs{\prog}+\nsamples H( L_S(\prog) ) + O(\log \nsamples + \log \overline{\xlength}(S)) $.  Taking an expectation over $S$ and recalling $\exv{H(L_S(\prog))}\leq H(\exv{L_S(\prog)}) = H(L(\prog))$ yields Corollary \ref{corollary:mdl_length}.  \removed{We emphasize that $\overline{\xlength}(S)$ need not be calculated by the program, and is hard-coded as an integer, as it doesn't depend on the input $x$ to the program.}
\end{proof}

The key ingredient to proving Theorem \ref{thm:short_hash_table} is a PRG based on a short seed length that can be used generate ``random'' binary function $\htable(x)$ with $\htable(x)\sim \Ber(\alpha)$, where $\alpha = \sum_i y_i / \nsamples$.  To make this precise, we consider a family of hash functions, indexed by a seed of length $\seedlength$, or in other words a seeded hash function of the form $\htable(\seed,x)$, where we will show that for every $\trainset$, there exists a seed such that $x\mapsto\htable(\seed,x)$ interpolates $\trainset$.  In Lemma \ref{lemma:hash_table_exists}, we show that such a seeded hash function exist and bound the required seed length.

\begin{lemma}
\label{lemma:hash_table_exists}
For all $\nsamples,\xlength\in\mathbb{N}$ and all $k\leq \nsamples$, and for
\begin{align*}
    \seedlength = \nsamples\cdot H\inparen{\nfrac{k}{\nsamples}} + \log \nsamples + \log\xlength + 1
\end{align*}
there exists a function $\htable:\B^{\seedlength} \times \B^{\xlength}\rightarrow\B$ such that for all distinct $x_1,\ldots,x_\nsamples \in \B^\xlength$ and all $y_1,\ldots,y_\nsamples\in\B$
with $\sum_i y_i=k$, there exists $\seed \in \{0,1\}^r$, such that $\forall_i \htable(\seed, x_i) = y_i$.
\end{lemma}
\begin{proof}
To show existence, we use the probabilistic method, showing that a random function $G:\B^{\seedlength} \times \B^{\xlength} \to \B$ has positive probability of satisfying the Lemma requirements.  Let $\alpha = \nfrac{1}{\nsamples}\cdot\sum_{i=1}^{\nsamples} y_i$. Choose $G$ at random by setting $G(\seed, x) = 1$ with probability $\alpha$ independently over all choices of $\seed$ and $x$. We will say that $G$ \textit{fails} if there exists some $\trainset = \inbraces{(x_i, y_i), \text{ for } i \in [\nsamples]}$ (with $x_i\neq x_j$ and $\sum_i y_i=k$) for which there is no corresponding $\seed$ such that interpolates $S$, i.e.~s.t.~$\forall_i G(\seed, x_i) = y_i$.

For a fixed $S$ and $\seed$, the probability $\seed$ interpolates $S$ is exactly $\alpha^{\alpha\nsamples}(1-\alpha)^{(1-\alpha)\nsamples} = 2^{-\nsamples \cdot H(\alpha)}$, and so the probability there is no seed that interpolates $S$ is $\left(1-2^{\nsamples H(\alpha)}\right)^{(2^r)}$. Taking a union bound over all $S$s:
\begin{align}
    \prvv{G}{G \text{ fails}} 
    &= \prvv{G}{\text{there exists } S \text{ on which } G \text{ fails}} \le \sum_{S} \prvv{G}{\text{there is no } \seed \text{ that interpolates } S} \notag \\
    &= \sum_{S} \inparen{1-2^{-\nsamples \cdot H(\alpha)}}^{2^{\seedlength}} 
    = \binom{\nsamples}{k}\binom{2^{\xlength}}{\nsamples}\cdot\inparen{1-2^{-\nsamples \cdot H(\alpha)}}^{2^{\seedlength}} \notag \\
    &< \expv{ \nsamples \ln(2) 
    + \nsamples\xlength\ln(2)-2^{-\nsamples \cdot H(\alpha)} \cdot 2^{\seedlength}} < \expv{0} = 1
\end{align}
where in the last inequality we plugged in the prescribed seedlength $r$. 
\removed{
We now show that our choice of $\seedlength$ makes the above expression strictly bounded above by $1$. Observe that
\begin{align*}
    \seedlength = \nsamples\cdot H(\alpha) + \logv{k\ln(\nsamples) + \nsamples\xlength\ln(2)} + 1 > \nsamples\cdot H(\alpha) + \logv{k\ln(\nsamples) + \nsamples\xlength\ln(2)}
\end{align*}
and so
\begin{align*}
    \expv{k\ln(\nsamples) + \nsamples\xlength\ln(2)-2^{-\nsamples \cdot H(\alpha)} \cdot 2^{\seedlength}} &< \expv{k\ln(\nsamples) + \nsamples\xlength\ln(2)-2^{\logv{k\ln(\nsamples) + \nsamples\xlength\ln(2)}}} = 1
\end{align*}}
Thus, $\prvv{G}{G \text{ fails}} > 0$, and there exists at least one function $G$ that satisfies Lemma \ref{lemma:hash_table_exists}.
\end{proof}

\natinotelater{This paragraph is a bit repetitive and should be shortened significantly}
Lemma \ref{lemma:hash_table_exists} establishes the {\em existence} of such a seeded hash function, but to actually use it in a short program, we not only need the seed to be short, but also the description of the function $\htable$ to be short.  Lemma \ref{lemma:hash_table_exists} does {\em not} provide such a description as it is non-constructive, and we are not aware of any explicit construction.  Fortunately, as we are not concerned with runtime, we can describe $\htable$ through a (short and explicit) program that enumerates over all $2^{2^{r+\xlength}}$ possible functions and picks the first one lexicographically. Using this ``explicit'' short programmatic description of $\htable$, we finish the proof of Theorem \ref{thm:short_hash_table}.

\begin{proof}[Theorem \ref{thm:short_hash_table}]
Let $\mathsf{GenHash}$ be a program that takes as input three integers $(\nsamples,k,\xlength)$, calculates $r$ based on them as defined in Lemma \ref{lemma:hash_table_exists}, enumerates over all functions $G:\B^{\seedlength} \times \B^{\xlength} \to \B$, and returns the lexicographically first function that satisfies Lemma \ref{lemma:hash_table_exists}.  
The size of the output of $\mathsf{GenHash}$, which is a huge lookup table, depends on its inputs, but the function description itself is fixed, with fixed length $\abs{\mathsf{GenHash}}$ (e.g.~$\abs{\mathsf{GenHash}}<1000$ in compressed Python or C++ with standard libraries).  Our program for interpolating $S$ is:
\begin{equation}
    \prog_S(x) = \mathsf{GenHash}(\nsamples,k,\mathsf{bitlength}(x))(\seed,x)
\end{equation}
where $k=\sum_i y_i$ and $\seed$ is the seed the interpolates $S$ using the lexicographically first function that satisfies Lemma \ref{lemma:hash_table_exists}.  This seed is hard-coded into the program.  The description length of program is thus $\abs{\mathsf{GenHash}}+\abs{\nsamples}+\abs{k}+\abs{\seed}+O(1)=r+2\log \nsamples + O(1)=\nsamples H(\nsamples/k) + 3 \log \nsamples + \log \xlength + O(1) $ (where here $\abs{a}$ is the description length of $a$).
\end{proof}

\paragraph{Tightness of Dependence on Disambiguation Prefix Length}  One might wonder whether it is possible to avoid, or improve, the dependence on $\xlength$ in Theorem \ref{thm:short_hash_table} or thus on $\overline{\xlength}$ in Corollary \ref{corollary:mdl_length}.  Unfortunately, this is not possible.  To see this, for any $\xlength$, we will construct a sample $S=\{(x_1,y_1),(x_2,y_2)\}$, with $x_1,x_2\in\B^\xlength$ that cannot be interpolated using any program of length less than $\log\xlength$.  We will do so by associating for every $x\in\B^\xlength$, a vector $\phi(x)\in\B^N$ consisting of the output of running each of the $N< \xlength$ programs of length $< \log\xlength$ on $x$.  I.e.~$\phi(x)[i]=\prog_i(x)$, where $\prog_i$ is the lexicographical $i$th program (and we can set $\phi(x)[i]=0$ if the program doesn't stop and output a valid value).  There are $2^\xlength$ different $x$s, but only $2^N<2^b$ possible $\phi(x)$, and so there must be two inputs $x_1=x_2$ with $\phi(x_1)\neq\phi(x_2)$, i.e.~such that no short program can distinguish between them.  The sample $S=\{(x_1,0),(x_2,1)\}$ can thus not be interpolated by any program of length less than $\log\xlength$.

%% file: paper_generalization.tex
\section{Generalization}
\label{sec:generalization}

After establishing in Section \ref{sec:interpolation} upper bounds on the length of $\mdl(S)$, we will now prove Theorems \ref{thm:agnostic_main} and \ref{thm:random_label_noise_main} by combining these with guarantees on the generalization error of any learning rule outputting a short program.  

\paragraph{Agnostic Guarantees and Proof of Theorem \ref{thm:agnostic_main}} The following generalization guarantee in terms of program length is a a tight version of the standard description-length based guarantee:
\begin{lemma}  
\label{lemma:agnostic_gen_mi}
For any distribution $\cD$, any interpolating learning rule $A$, and any sample size $\nsamples$: 
$$-\log\left(1-\exvv{\trainset \sim \cD^{\nsamples}}{L(A(\trainset))}\right) 
\le \frac{I(S;A(S))}{\nsamples}
\le  \frac{\exv{\abs{A(\trainset)}}}{\nsamples}.$$
\end{lemma}
We can obtain a high probability version of the Lemma \ref{lemma:agnostic_gen_mi}, in terms of $\sup \abs{A(S)}$, using a union bound over all short programs. This is also a special case of the PAC-Bayes Bound \cite{david_pacbayes}, noting that $\kld{0}{\beta}=-\log(1-\beta)$. \citet{raginsky2016information,russo2019much,xu2017information} obtain similar (and more general) bounds, but bound $L(A)^2$ rather than $-\log(1-L(A))$ on the left-hand side.  Since in our case the right-hand-side will not vanish, this precision is significant---a tight bound, even up to constant factors, is essential for obtaining tempered overfitting guarantees.  

For completeness, we provide a proof of Lemma \ref{lemma:agnostic_gen_mi}, which we will also use as a basis for a more refined analysis in Lemma \ref{lemma:mutual_info_uc}.  The proof captures the following information argument: if the error rate of $A(S)$ outside $S$ is very different than the error rate of $A(S)$ inside $S$ (which is zero since $A$ interpolates), this signal, based on $A(S)$, on which $(x,y)$ are in the sample (i.e.~can be used to predict membership in $S$ based on the output of $A(S)$), and thus a lower bound on $I(S;A(S))$.

\begin{proof}[Proof of Lemma \ref{lemma:agnostic_gen_mi}]  We denote $\ProgW=A(S)$ and think of it as a random variable.  We have: 
\begin{align}\label{eq:inf1}
    \exv{\abs{\ProgW}} \geq H(\ProgW) \geq I(S;\ProgW) \geq \sum_i I((X_i,Y_i);\ProgW) = \nsamples I((X_1,Y_1);\ProgW) 
\end{align}
where the first inequality is Shannon's source coding bound and the third inequality is due to the independence of $(X_i,Y_i)$ (Lemma \ref{lemma:mutual_info_add} in the Appendix).  To analyze $I((X_1,Y_1);\ProgW)$, we rely on the variational bound $I(A;B)\geq \exvv{A,B\sim p}{\log\frac{dq(A|B)}{dp_A(A)}}$, where $p(A,B)$ is the true joint distribution, with marginal $p_A$, and $q(A|B)$ is any proposed conditional distribution (Lemma \ref{lemma:kl_conditional}).  In our case, we use the proposal distribution $q(X_1,Y_1|\ProgW)$ defined as:
\begin{equation}
    dq(x,y|\progw) = \tfrac{1}{Z_w} \indicator{\progw(x)=y} dp(x,y) 
\end{equation}
where $p(x,y)$ is the true marginal over $X_1,Y_1$ (i.e.~the source distribution $\cD$), and\\
$Z_w = \exvv{X,Y\sim p}{\indicator{\progw(X)=Y}} = 1-L(\progw)$.  
This proposal distribution amounts to bounding the mutual information by the information $\ProgW$ tells us about $(X_1,Y_1)$ by telling us that $(X_1,Y_1)$ satisfies $\ProgW(X)=Y$ (since $\ProgW=A(S)$ interpolates the training points).  We now calculate:
\begin{multline}\label{eq:var1}
I((X_1, Y_1); \ProgW) 
\ge \exvv{S}{\logv{\frac{dq(X_1,Y_1|\ProgW)}{dp(X_1,Y_1)}}} 
= \exvv{S}{\logv{\frac{\indicator{\ProgW(X_1)=Y_1}}{Z_\ProgW}}} \\
= \exvv{\trainset}{\log \frac{1}{Z_\ProgW}} 
= \exvv{\trainset}{-\logv{1-L(\ProgW)}} 
\ge -\logv{1-\exvv{\trainset}{L(\ProgW)}}
\end{multline}
Recalling that $\ProgW=A(S)$, and combining \eqref{eq:inf1} and \eqref{eq:var1} we obtain the statement of Lemma \ref{lemma:agnostic_gen_mi}.
\end{proof}

\noindent Plugging Corollary \ref{corollary:mdl_length} into Lemma \ref{lemma:agnostic_gen_mi} we have for any $\prog$, 
\begin{equation}\label{eq:from41}
    \exv{L(\mdl(\trainset))} \le 1 - 
2^{-H(L(\prog)) - \inparen{\abs{\prog} + O\inparen{\log \nsamples + \log\overline{\xlength}(\nsamples)}}/\nsamples}.
\end{equation}  As $\nsamples\rightarrow\infty$, the right hand side converges to $\gagg(L(\prog))=1-2^{-H(L(\prog))}$.  We now use the following inequality (proved as Lemma \ref{lemma:ag_uc_mathematica} in the Appendix) to bound the convergence rate to $\gagg(L(\prog))$:
\begin{equation}\label{eq:frommathematica}
    \textrm{For $C\geq 0$ and $0\leq \alpha\leq 1$,} \quad 1-2^{-H(\alpha)-C} \leq \gagg(\alpha)+C 
\end{equation}
Combining \eqref{eq:from41} and \eqref{eq:frommathematica} yields Theorem \ref{thm:agnostic_main}. \qed

\paragraph{Tightness of Agnostic Generalization Guarantee} Before continuing, we note that Lemma \ref{lemma:agnostic_gen_mi} is tight, and we cannot hope to get a better guarantee solely in terms of program length.  To see this, consider a distribution $\cD$ where $X$ is i.i.d.~uniform bits, and $Y=0$.   Although this is a very easy distribution to interpolate, consider, for any $0<\alpha\leq 0.5$, an interpolating learning rule $A$ that searches for a random function $\htable(x)$ that interpolates the data, where $\htable(x)\sim\ber{\alpha}$ independently for all $x$.  Using arguments similar to the proof of Theorem  \ref{thm:short_hash_table}, we can calculate that the probability of such a function interpolating the data is $(1-\alpha)^\nsamples$, and we can therefore encode such a function using $\exv{\abs{A(S)}}=O( \nsamples \log(1-\alpha) + \log \nsamples )$ bits.  For large $\nsamples$, the right-hand-side of Lemma \ref{lemma:agnostic_gen_mi} is therefore $\log(1-\alpha) + o(1)$.  This is tight since $L(A)=\alpha=1-2^{\log(1-\alpha)}$.

\paragraph{Random Label Noise and Proof of Theorem \ref{thm:random_label_noise_main}}

Although Lemma \ref{lemma:agnostic_gen_mi} is worst-case optimal, we show a tighter generalization guarantee for well specified distributions with independent label noise:
\begin{lemma}
\label{lemma:mutual_info_uc} 
For any source distribution $\cD$ such that $Y|X=\hstar(X)\oplus\Ber(\Lstar)$, any learning rule $A(S)$ returning an interpolating program, and any sample size $\nsamples$, we have
\begin{align}
    \kld{\Lstar}{\exvv{S}{\prvv{X \sim \cD}{A(S)(X)\neq\hstar(X)}}} \le \frac{\exvv{\trainset}{\abs{A(\trainset)}} - \nsamples \cdot H\inparen{\Lstar}}{\nsamples} \tag{A}\label{eq:wellA}
\end{align}
and therefore
\begin{align}
    \abs{ L(A) \removed{\exvv{\trainset}{L(A(\trainset))}} - 2\Lstar(1-\Lstar)} \leq &O\inparen{\frac{\exvv{\trainset}{\abs{A(\trainset)}} - \nsamples H(\Lstar)}{\nsamples} + \sqrt{\Lstar\cdot\frac{\exvv{\trainset}{\abs{A(\trainset)}} - \nsamples  H(\Lstar)}{\nsamples}}} \tag{B}\label{eq:wellB}
\end{align}
\end{lemma}
The proof is again information theoretic based on the following intuition: The agreement rate of $A(S)$ with $\hstar$ inside $S$ is exactly $\Lstar$.  If the agreement rate outside $S$ differs significantly, we can use it to construct a predictor for which $x$s are in $S$ and thus the output of $A(S)$ has information about the $X_i$s.  But $A(S)$ needs at least $m H(\Lstar)$ bits of information just for encoding the noise on the labels, and so if it's description length is not much more than $m H(\Lstar)$, it can't also contain information about which $x$s are in $S$ (i.e.~it doesn't have enough information capacity for also memorizing anything about the $X_i$s).

\begin{proof}
Denoting $\ProgW=A(S)$ as before, we have
\begin{align}\label{eq:IXYW}
    \exv{\abs{\ProgW}} \geq \nsamples I(X_1,Y_1;\ProgW) = \nsamples (X_1;\ProgW) + \nsamples I(Y_1;\ProgW|X_1)
\end{align}
where the inequality is the same as in the proof of Lemma \ref{lemma:agnostic_gen_mi}.  We evaluate: 
\begin{align}\label{eq:IY}
    I(Y_1;\ProgW|X_1) = H(Y_1|X_1) - H(Y_1|\ProgW,X_1) = H(\Lstar)-0 = H(\Lstar)
\end{align}
where in the second equality, the first term follows since $Y_1|X_1\sim\ber{\Lstar}$ based on the noise model, and the second is because $Y_1=\ProgW(X_1)$ is a deterministic function of $\ProgW,X_1$.  

In order to bound $I(X_1;\ProgW)$, it will be convenient to define $\tilde{\ProgW}$, which is a deterministic function of $\ProgW$ (and hence also a random variable) with $\tilde{\ProgW}(x)=\ProgW(x)\oplus\hstar(x)$ (recall $\hstar$ is fixed and deterministic here).  We will also denote
$\Ltilde = \exv{\prv{\ProgW(X)\neq\hstar(X)}} = \exv{\tilde{\ProgW}(X)}$ the disagreement probability we want to bound.  Now, to bound $I(X_1;\ProgW)$, we will use the same variational bound, this time with the proposal distribution:
\begin{align}
    dq_{X|\ProgW}(x|\progw) = \frac{1}{Z_\progw} \frac{p_{\ber{\Lstar}}(\tilde{\progw}(x))}{p_{\ber{\Ltilde}}(\tilde{\progw}(x))} dp(x) \label{eq:condition_distro}
\end{align}
where $p_{\ber{\alpha}}(0) =1-\alpha, p_{\ber{\alpha}}(1) =\alpha$ is the Bernoulli p.m.f., and again $p(x)$ is the true (population) marginal. This proposal distribution is the best we can do solely in terms of $\tilde{\progw}(x)$, since we know that inside $S$ we have $\tilde{\ProgW}(X_1)=\ProgW(X_1) \oplus \hstar(X_1) = Y_1 \oplus \hstar(X_1) \sim \ber{\Lstar}$ while for a random $X$,  $\tilde{\ProgW}(X)=\ProgW(X) \oplus \hstar(X)\sim\ber{\Ltilde}$, by definition of $\Ltilde$.  We can calculate the partition function:
\begin{align}
    Z_\progw = \exvv{X\sim p}{\frac{p_{\ber{\Lstar}}\left(\tilde{\progw}(X)\right)}{p_{\ber{\Ltilde}}\left(\tilde{\progw}(X)\right)}} 
    = \prvv{X}{\tilde{\progw}(X)=1} \cdot \frac{\Lstar}{\Ltilde} + \prvv{X}{\tilde{\progw}(X) =0} \cdot \frac{1-\Lstar}{1-\Ltilde} &\label{eq:zm}
\end{align}
Taking an expectation over $\ProgW$, we have $\exv{Z_{\ProgW}}=    \frac{\Lstar}{\Ltilde} \cdot \Ltilde + \frac{1-\Lstar}{1-\Ltilde} \cdot (1-\Ltilde) = 1$.
\removed{
\begin{align}
\exvv{S}{ Z_{\ProgW} } = 
    \exvv{\ProgW' \sim p_\ProgW}{Z(\ProgW')} &= \exvv{\ProgW' \sim p_\ProgW}{\prvv{X_i \sim p_X}{\ProgW'(X_i) \neq \hstar(X_i)} \cdot \frac{\Lstar}{\Ltilde} + \prvv{X_i \sim p_X}{\ProgW'(X_i) = \hstar(X_i)} \cdot \frac{1-\Lstar}{1-\Ltilde}} & \nonumber \\
    &= \frac{\Lstar}{\Ltilde} \cdot \exvv{\ProgW' \sim p_\ProgW}{\prvv{X_i \sim p_X}{\ProgW'(X_i) \neq \hstar(X_i)}} + \frac{1-\Lstar}{1-\Ltilde} \cdot \exvv{\ProgW' \sim p_\ProgW}{\ProgW'(X_i) = \hstar(X_i)} & \nonumber \\
    &=
    \frac{\Lstar}{\Ltilde} \cdot \Ltilde + \frac{1-\Lstar}{1-\Ltilde} \cdot (1-\Ltilde) = 1 & \label{eq:exv_z}
\end{align} }
Applying the variational bound (Lemma \ref{lemma:kl_conditional}) we have:
\begin{multline}
    I(X_1; \ProgW) 
    \ge \exvv{X_1,\ProgW}{\logv{\frac{dq(X_1|\ProgW)}{dp(X_1)}}} 
    = \exvv{X_1,\ProgW}{\logv{\frac{p_{\ber{\Lstar}}(\tilde{\ProgW}(X_1))}{p_{\ber{\Ltilde}}(\tilde\ProgW(X_1))}\cdot\frac{1}{Z_\ProgW}}} \\
    \geq \exvv{\tilde\ProgW(X_1) \sim \Ber(\Lstar)}{\logv{\frac{p_{\ber{\Lstar}}(\tilde\ProgW(X_i))}{p_{\ber{\Ltilde}}(\tilde\ProgW(X_i))}}}  - \logv{\exvv{\ProgW}{Z_\ProgW}} = \kld{\Lstar}{\Ltilde}
\removed{    
    &\ge \exvv{X_i,\ProgW \sim p_{X_i,\ProgW}}{\logv{\frac{p_{\ber{\Lstar}}(\ProgW(X_i))}{p_{\ber{\Ltilde}}(\ProgW(X_i))}}} + \logv{\exvv{X_i,\ProgW \sim p_{X_i,\ProgW}}{\frac{1}{Z(\ProgW)}}}\nonumber  & \text{(Concavity of $\log$)} \\
    &= \exvv{\ProgW(X_i) \sim \Ber(\Lstar)}{\logv{\frac{p_{\ber{\Lstar}}(\ProgW(X_i))}{p_{\ber{\Ltilde}}(\ProgW(X_i))}}} + \log 1 = \kld{\Lstar}{\Ltilde} & \text{(Equation \ref{eq:exv_z})}}
    \label{eq:old_lemma_44} 
\end{multline}
Where the inequality is due to Jensen on the second term, and we then use $\exv{Z_\ProgW}=1$ to cancel it. 

Plugging in \eqref{eq:old_lemma_44} and \eqref{eq:IY} into \eqref{eq:IXYW} yields part \eqref{eq:wellA} of the Lemma.  To obtain part \eqref{eq:wellB},  we first use the inequality $\abs{\beta-\alpha}\leq 2 \kld{\alpha}{\beta} + \sqrt{2\alpha \kld{\alpha}{\beta}}$ (Lemma \ref{lemma:strong_pinsker}) to obtain $\abs{\Ltilde-\Lstar}\leq 2 \rhs + \sqrt{2 \Lstar \rhs}$, where $\rhs$ is the right hand side of part \eqref{eq:wellA}.   And since $L(A)=\Ltilde(1-\Lstar)+(1-\Ltilde)\Lstar$, we have $\abs{L(A)-2\Lstar(1-\Lstar)}=(1-2\Lstar)\abs{\Ltilde-\Ltilde}\leq \abs{\Ltilde-\Ltilde}$.  Combining the two inequalities yields part \eqref{eq:wellB}.
\end{proof}



\noindent Plugging in Corollary \ref{corollary:mdl_length} into Part \eqref{eq:wellB} of Lemma \ref{lemma:mutual_info_uc} yields Theorem \ref{thm:random_label_noise_main}. \qed

%% file: paper_tightness.tex
\section{Tightness and Discussion}
\label{sec:discussion}

For MDL interpolation in the presence of random label noise, we provide a precise characterization of the effect of overfitting.  In this case, unlike the optimally tuned SRM, which converges to the Bayes optimal predictor, the interpolating $\mdl$ predictor will converge to sampling from the posterior, yielding up to twice the Bayes error.  This is similar to the behavior of a 1-nearest-neighbor rule (although the actual predictions will of course be very different), the observed behavior of certain neural networks \cite{nakkiran2020distributional}, and perhaps kernels \cite{mallinar2022benign}.  This is a ``tempered'' behavior, where for any non-trivial Bayes error $0\leq \Lstar \leq 0.5$, the limiting MDL error $\Lstar<\gsamp(\Lstar)<0.5$ is strictly worse than Bayes, but still provides an edge over random guessing.

In the more general agnostic case, we give only an upper bound, depicted in Figure \ref{fig}.  Although strictly worse than the sampling behavior with random label noise, this behavior is still tempered (Corollary \ref{cor:limit}): if some computable function has non-trivial error $L(\prog)<0.5$, the optimally tuned $SRM$ will converge to at most this error, and MDL might suffer due to overfitting, but we will still yield (as $\nsamples\rightarrow\infty$) an edge over random guessing and error at most $\gagg(L(\prog))<0.5$.  

\paragraph{Tempered Behavior with Finite Samples} An important feature of our results is that we do not look only at the asymptotic behavior, but ask also about the effect of overfitting with a finite number of samples, and how we compare to the finite-sample agnostic SRM guarantee \eqref{eq:srm}.  In particular, with finite $\nsamples$, the competitor $\prog$ with which we want to compete (i.e.~the one minimizing the right hand side of \eqref{eq:srm}) might be different and depend on $\nsamples$.  Indeed,  our finite sample agnostic guarantee (Theorem \ref{thm:agnostic_main} shows that we can compete with the $\nsamples$-dependent $\prog$ with which SRM competes, with a ``tempered'' effect on the error.  This is similar in spirit to the study of how minimum norm interpolation can adapt the approximation error to the sample complexity as recently studied by \citet{misiakiewicz2022spectrum,xiao2022precise}.

\paragraph{Tightness of Agnostic Guarantee} One might ask whether our agnostic upper bound is tight and whether it is possible to identify its exact behavior.

First, we point out that $\mdl$ might yield limiting error anywhere between the Bayes error and the error of the sampling predictor, i.e.~anywhere in the red region between the Bayes optimal line and sampling curve in Figure \ref{fig}.  \removed{More precisely, for any Bayes error $0<\Lstar<0.5$, and any error $\Lstar \leq L_{\mdl} \leq 2\Lstar(1-\Lstar)$, we can describe a distribution where $L(MDL) \stackrel{\nsamples\rightarrow\infty}{\rightarrow}L_{\mdl}$.}  To see this, consider a source distribution where $X[1]\sim\ber{\alpha}$, the remaining bits of $X$ are i.i.d.$\ber{0.5}$, and $Y=X[2]$ if $X[1]=0$, but $Y=\ber{\beta}$ if $X[1]=1$.  It is easy to verify that $\Lstar=\alpha\beta$ while $L(MDL) \stackrel{\nsamples\rightarrow\infty}{\rightarrow}L_{\mdl} = 2\alpha\beta(1-\beta)$, which allows us to get any $0\leq \Lstar \leq L_{\mdl} \leq \gsamp(\Lstar) \leq 0.5$ by varying $\alpha$ and $\beta$.  This is the same sampling behavior and same asymptotic error that will be reaches by other sampling-type over-fitting predictors, such as 1-nearest-neighbor.

We do not know whether there are source distributions for which $\mdl$ will yield errors above the sampling curve $\gsamp$ (the green region in Figure \ref{fig}), or whether the difference between $\gsamp$ and $\gagg$ is due to a looseness in Theorem \ref{thm:agnostic_main}.  In Sections \ref{sec:interpolation} and Section \ref{sec:generalization} we argued that the description length bound in Corollary \ref{corollary:mdl_length} and the generalization bound in terms of program length in Lemma \ref{lemma:agnostic_gen_mi} are tight.  This implies our proof technique, which separately asks what length programs we need to consider and then uses what is essentially a uniform generalization guarantee for all short programs, cannot improve beyond Theorem \ref{thm:agnostic_main} (in the agnostic case).  But although this proof technique cannot be improved, it is possible that by analyzing specific properties of the $\mdl$, it is possible to significantly strengthen \ref{thm:agnostic_main}, perhaps replacing $\gagg$ with $\gsamp$ also in the agnostic case, and we leave this as an open question.

It is useful to note that if the posterior $\eta(x)=P(Y=1|x)$ is computable, $\mdl$ should also converge to a sampling classifier and yield limiting error $L(\mdl) \stackrel{\nsamples\rightarrow\infty}{\rightarrow}L_{\mdl} \leq \gsamp(L^*)$ where $\Lstar$ is the Bayes error.  In fact, we suspect it is possible to generalize Theorem \ref{thm:random_label_noise_main} to show that:
\begin{equation}
    \exvv{X}{\kld{\eta(X)}{\prvv{S}{\mdl(S)(X)=1}}} \leq  \frac{\abs{\eta} + O(\log \nsamples + \log \bar\xlength(\nsamples) )}{\nsamples}, 
\end{equation}
where $\abs{\eta}$ is the description length of the (computable) posterior $\eta$. 
This is a more general situation than random label noise added to a computable Bayes optimal predictor, where $\eta(x)=\Lstar + \prog^{\star}(1-2\Lstar)$.  The scenario where MDL might yield error above $\gsamp(L(\prog^{\star}))$, is thus when the Bayes predictor $\prog^{\star}(x) = \sign(\eta(X)-0.5)$ is computable, but the posterior $\eta(x)$ itself is not.  Even without getting to non-computability, we can consider a situation where the Bayes optimal predictor has a very short description, but the posterior requires a much longer program, and ask whether this would result in large gaps between the optimally balanced SRM and the interpolating $\mdl$.

\paragraph{Different Notions of Description Length or Different Inductive Bias}  We considered $\mdl$ learning in the Turing or Kolmogorov sense, i.e.~by minimizing program length.  This is arguably the most general notion, if we would like the learned predictor to actually be computable.  Still, one can instead think more abstractly of logical descriptions that allow quantifiers over infinite domains.  Our results hold also in these more general settings, or any other notion that subsumes or extend Turing computation.  More specifically, all we require from the notion of description is that we can describe ``lexicographically first function satisfying Lemma \ref{lemma:hash_table_exists}.''  

Alternatively, one might consider more limited notions of description, e.g.~limiting to only programs with short runtime, and considering the learning rule\footnote{While still abstract, the learning rule $\mathsf{MinRuntime}$ is more useful as a reference universal rule, since we want our predictor to not merely be computable, but also be tractable with reasonable runtime \cite{valiant1984theory}.  Additionally, $\mathsf{MinRuntime} \in \mathrm{NP}$, and for all we know might be poly-time computable, unlike $\mdl$ which is uncomputable.} $\mathsf{MinRuntime}$ that selects the program with the minimal (worst case) runtime that interpolates $S$.  Or almost equivalently (up to some polynomial relationship), limiting to neural networks and considering the learning rule $\mathsf{MinNetwork}$ which returns the neural network\footnote{More formally, we fix the activation function, e.g.~to ReLU activation, and search over all architecture graphs and all edge weights.} with the minimal number of edges that interpolates the training set.  Our analysis does not apply to $\mathsf{MinRuntime}$ or $\mathsf{MinNetwork}$ since the short program we construct has double-exponential runtime.  
An explicit and efficiently computable pseudo-random generator, generating $N$ bits that are (approximately) $\nsamples$-way independent and marginally $\ber{\alpha}$ using a seed length of $\nsamples\cdot H(\alpha) + O( \log m + \log\log N)$ (or even a worse dependence on $N$), would allow extending our results also to min-runtime or min-size-neural-net interpolation.   

More generally, our analysis can be viewed as providing a sufficient condition on an inductive bias $c(\prog)$ such that minimum-$c(\prog)$ interpolation exhibit tempered overfitting:  roughly speaking, as long as the inductive bias allows us to encode ``random function'' with capacity (i.e.~the capacity of the sublevel set of $c(\cdot)$ containing these random functions) not much larger than the capacity of the set of such ``random functions'', it should be the case that minimum-$c(\prog)$ interpolation is tempered in the sense of Theorems \ref{thm:agnostic_main} and \ref{thm:random_label_noise_main}.

\paragraph{Tightness of Dependence on the Disambiguation Prefix Length}  Another open technical question is whether the mild dependence on the quenched disambiguation prefix length in Theorems \ref{thm:agnostic_main} and \ref{thm:random_label_noise_main} is necessary.  Again, we argue that it is necessary for bounding the description length, and so for our proof technique.  But the examples which require long programs due to the differences between instances being hidden in far-away and hard-to-describe bits, do not show these long programs do not generalize well.  We do not know and leave it open whether the dependence on $\overline{\xlength}$ in Theorems \ref{thm:agnostic_main} and \ref{thm:random_label_noise_main} is necessary, or whether different techniques and specific analysis of the MDL can avoid these.

\removed{
\paragraph{Descriptions vs Programs} In this paper we formalized ``descriptions'' as programs in a Touring complete language.  One might consider also more general, and uncomputable, notions of descriptions, e.g.~involving quantifiers over infinite domains.  Our results hold also for such notions, as our only requirement is tha}


\paragraph{Summary} With the growing interest in noisy interpolation learning, and obtaining an understanding and characterization of the ``benignness'' and/or harm of overfitting, we find it instructive to consider what is perhaps the most basic and fundamental learning principal, with roots going back to the first discussions of machine learning and inductive inference \cite{solomonoff1960preliminary}.  We hope that our study will help direct our search for the fundamental principles of what ``makes'' overfitting benign or catastrophic.  We would also like to see our tempered finite sample agnostic guarantee (Theorem \ref{thm:agnostic_main}) as a template for studying  how overfitting compares with the optimally balanced approached (the SRM guarantee of \eqref{eq:srm} in our case), instead of focusing on comparing the asymptotic behavior and seeking consistency, which is frequently less relevant for learning.

%% file: appendix.tex
\section{Information Theoretic Identifies and Inequalities}
\label{app:info_theory}

We present and either cite or prove several identities and inequalities we use in our proofs.

\begin{lemma}[Chain Rule of Mutual Information; see p. 42 of \cite{cover2006elements}]
\label{lemma:mutual_info_chain}
For any random variables $A_1, A_2$, and $B$:
\begin{align*}
    I((A_1,A_2); B) = I(A_2; B | A_1) + I(A_1; B)
\end{align*}
\end{lemma}

\begin{lemma}
\label{lemma:kl_conditional}
Let $A$ and $B$ be any two random variables with associated marginal distributions $p_A$, $p_B$, and joint $p_{A,B}$. Let $q_{A|B}$ be any conditional distribution (i.e.~such that for any $b$, $q_{A|B}(\cdot,b)$ is a normalized non-negative measure). Then:
\begin{align*}
    I(A;B) \ge \exvv{A,B \sim p_{A,B}}{\logv{\frac{dq_{A|B}(A|B)}{dp_A(A)}}}
\end{align*}
\end{lemma}
\begin{proof}
\removed{
We use the non-negativity of the the KL divergence to write
\begin{align*}
    0 \le \exvv{B \sim p_B}{\kld{p_{A|B}}{q_{A|B}}} &= \exvv{B \sim p_B}{\exvv{A \sim p_{A | B}}{\logv{\frac{dp_{A|B}(A|B)}{dq_{A|B}(A|B)}}}} \\
    &= \exvv{A,B\sim p_{A,B}}{\logv{\frac{dp_{A|B}(A|B)}{dp_A(A)}}+\logv{\frac{dp_A(A)}{dq_{A|B}(A|B)}}} \\
    &= \kld{p_{A|B}}{p_A} - \exvv{A,B \sim p_{A,B}}{\logv{\frac{dq_{A|B}(A|B)}{dp_A(A)}}} \\
    &= I(A;B) - \exvv{A,B \sim p_{A,B}}{\logv{\frac{dq_{A|B}(A|B)}{dp_A(A)}}}
\end{align*}
}
The proof essentially uses the chain rule for KL-divergence:
\begin{align}
I(A;B) & = \kld{p_{A|B}}{p_A}  = \exvv{A,B\sim p_{A,B}}{\logv{\frac{dp_{A|B}(A|B)}{dp_A(A)}}} \\
& = \exvv{A,B\sim p_{A,B}}{\logv{\frac{dp_{A|B}(A|B)}{dp_A(A)}
\cdot \frac{dq_{A|B}(A|B)}{dq_{A|B}(A|B)} }} \\
&= \exvv{A,B\sim p_{A,B}}{\logv{\frac{dq_{A|B}(A|B)}{dp_A(A)}}}  
+ \exvv{A,B\sim p_{A,B}}{\logv{\frac{dp_{A|B}(A|B)}{dq_{A|B}(A|B)}}} \\
& = \exvv{A,B\sim p_{A,B}}{\logv{\frac{dq_{A|B}(A|B)}{dp_A(A)}}}
+ \exvv{B \sim p_B}{\kld{p_{A|B}}{q_{A|B}}} \\
&\geq \exvv{A,B\sim p_{A,B}}{\logv{\frac{dq_{A|B}(A|B)}{dp_A(A)}}}
\end{align}
where the inequality follows from the non-negativity of the KL divergence.
\end{proof}

\begin{lemma}
\label{lemma:mutual_info_add}
Let $A_1, A_2, B$ be random variables where $A_1$ and $A_2$ are independent. Then 
\begin{align*}
    I((A_1,A_2); B) \ge I(A_1; B) + I(A_2; B)
\end{align*}\end{lemma}
\begin{proof}
We use Lemma \ref{lemma:kl_conditional} with the conditional distribution $q_{A_1,A_2|B} = p_{A_1|B} \cdot p_{A_2|B}$:
\begin{align}
    I((A_1,A_2);B) &\ge \exvv{A_1,A_2,B}{\logv{\frac{dp_{A_1|B}(A_1|B) \cdot dp_{A_2|B}(A_2|B)}{dp_{A_1,A_2}(A_1,A_2)}}} \\
    &= \exvv{A_1,B}{\logv{\frac{dp_{A_1|B}(A_1|B)}{dp_{A_1}(A_1)}}} + \exvv{A_2,B}{\logv{\frac{dp_{A_2|B}(A_2|B)}{dp_{A_2}(A_2)}}} \\
    &= I(A_1;B) + I(A_2;B) \notag \qedhere
\end{align}
\end{proof}

\begin{lemma}
\label{lemma:ag_uc_mathematica} 
For $C\geq 0$ and $0\leq \alpha\leq 1$, $1-2^{-H(\alpha)-C} \leq \gagg(\alpha)+C$.
\end{lemma}
\begin{proof}
We first prove that for all $\alpha, \beta \in (0,1)$ such that $\beta \ge 1-2^{-H(\alpha)}$, we have
\begin{align}
    \logv{\frac{1}{1-\beta}} - H(\alpha) \ge \beta - \inparen{1-2^{-H(\alpha)}}\label{eq:ag_uc_mathematica_interm}
\end{align}
Let $g(a) \coloneqq -\logv{1-a}-a$. Notice that the derivative of $g(a)$ is $g'(a) = -1 + \inparen{\ln(2)-a\ln(2)}^{-1}$.

First, we show that for all $a \in (0,1)$, we have $g(a) \ge 0$. We do so by showing that $g(0) = 0$ and that $g(a)$ is increasing on $a \in (0,1)$. It is easy to see that equality is achieved at $a=0$, so it is enough to show that $g'(a) \ge 0$ for all $a \ge 0$. This follows immediately since $\ln(2) < 0$. 

Next, we analyze $g(\beta) - g(1-2^{-H(\alpha)})$. Since $g(\cdot)$ is nonnegative and increasing, and since we assume $\beta \ge 1-2^{-H(\alpha)}$, we have $g(\beta) - g(1-2^{-H(\alpha)}) \ge 0$. Inequality \ref{eq:ag_uc_mathematica_interm} follows from expanding the definition of $g(\cdot)$ and rearranging. 

We now turn to proving the statement of Lemma \ref{lemma:ag_uc_mathematica}. Set $\beta = 1-2^{-H(\alpha)-C}$ and notice that
\begin{equation*}
    1-2^{-H(\alpha)-C} = \beta \le  \logv{\frac{1}{1-\beta}} - H(\alpha) + \gagg(\alpha) = \inparen{1-2^{-H(\alpha)}} + C = \gagg(\alpha) + C 
    \qedhere
\end{equation*}
\end{proof}

\begin{lemma}[Following \citet{david_pacbayes}, page 4]
\label{lemma:strong_pinsker}
Let $\alpha, \beta \in [0,1]$. Then $$\abs{\beta-\alpha} \le \sqrt{2\alpha\kld{\alpha}{\beta}} + 2\kld{\alpha}{\beta}.$$
\end{lemma}
\begin{proof}[Proof of Lemma \ref{lemma:strong_pinsker}]
First, consider the case where $\beta \ge \alpha$. We will show
\begin{align}
     \kld{\alpha}{\beta} - \frac{\inparen{\beta-\alpha}^2}{\inparen{2\ln 2}\beta} \ge 0\label{eq:strong_pinsker_beta_big_s}
\end{align}
To do so, notice that at $\beta = \alpha$, we achieve equality. It is now enough to show that the first derivative of the LHS of Inequality \ref{eq:strong_pinsker_beta_big_s} with respect to $\beta$ is always nonnegative. Notice that the first derivative of the LHS of Inequality \ref{eq:strong_pinsker_beta_big_s} with respect to $\beta$ is
\begin{align}
    \frac{1}{\ln 2}\inparen{\frac{\beta-\alpha}{\beta\inparen{1-\beta}} + \frac{\beta-\alpha}{\beta} + \frac{\inparen{\beta-\alpha}^2}{2\beta^2}}\label{eq:strong_pinsker_beta_big_derivative}
\end{align}
Since $\beta \ge \alpha$, we have $\nfrac{(1-\alpha)}{(1-\beta)} - \nfrac{\alpha}{\beta} \ge 0$. The other terms of Equation \ref{eq:strong_pinsker_beta_big_derivative} are clearly nonnegative when $\beta \ge \alpha$, which establishes Inequality \ref{eq:strong_pinsker_beta_big_s}.

Now, consider the following slight weakening of Inequality \ref{eq:strong_pinsker_beta_big_s}:
\begin{align}
    \kld{\alpha}{\beta} - \frac{\inparen{\beta-\alpha}^2}{2\beta} \ge 0\label{eq:strong_pinsker_beta_big}
\end{align}
We rearrange and obtain a quadratic in $\beta$:
\begin{align}
    0 \ge \beta^2 - 2\beta\inparen{\alpha+\kld{\alpha}{\beta}}+\alpha^2\label{eq:strong_pinsker_beta_big_quad}
\end{align}
Using the quadratic formula to solve for $\beta$ and subadditivity of $\sqrt{\cdot}$ on Inequality \ref{eq:strong_pinsker_beta_big_quad}, we have
\begin{align}
    \beta\le\alpha + \kld{\alpha}{\beta} + \sqrt{2\alpha\kld{\alpha}{\beta}+\kld{\alpha}{\beta}^2} \le \sqrt{2\alpha\kld{\alpha}{\beta}}+2\kld{\alpha}{\beta}
\end{align}
which is our upper bound on $\beta-\alpha$. 

For $\beta \le \alpha$ (from which we desire a lower bound on $\beta-\alpha$), we will show
\begin{align}
    \kld{\alpha}{\beta} - \frac{\inparen{\beta-\alpha}^2}{2\alpha\inparen{1-\alpha}} \ge 0\label{eq:strong_pinsker_beta_small}
\end{align}
As before, notice that equality holds when $\beta = \alpha$. It is now enough to show that the derivative of the LHS of Inequality \ref{eq:strong_pinsker_beta_small} is nonpositive whenever $0 \le \beta \le \alpha$. Notice that the first derivative of the LHS of Inequality \ref{eq:strong_pinsker_beta_small} with respect to $\beta$ is
\begin{align}
    \frac{1}{\ln 2}\inparen{\inparen{\alpha-\beta}\inparen{\frac{1}{\alpha\inparen{1-\alpha}}-\frac{1}{\beta\inparen{1-\beta}}}}\label{eq:strong_pinsker_beta_small_derivative}
\end{align}
We easily verify that Equation \ref{eq:strong_pinsker_beta_small_derivative} is nonpositive wherever $\beta \le \alpha$, which completes the proof of Inequality \ref{eq:strong_pinsker_beta_small}.

We now solve for $\beta$ by rearranging Inequality \ref{eq:strong_pinsker_beta_small}, yielding
\begin{align}
    \beta \ge \alpha - \sqrt{2\alpha(1-\alpha)\kld{\alpha}{\beta}} \ge \alpha - \sqrt{2\alpha\kld{\alpha}{\beta}} \ge \alpha - \sqrt{2\alpha\kld{\alpha}{\beta}} - 2\kld{\alpha}{\beta}
\end{align}
This yields the lower bound on $\beta-\alpha$ and concludes the proof of Lemma \ref{lemma:strong_pinsker}.
\end{proof}

In Lemma \ref{lemma:removing_duplicates}, we show that removing $k^+$ repeated examples with $y_i=1$,and $k^-$ repeated examples with $y_i=0$ only reduces the program length guaranteed by Theorem \ref{thm:short_hash_table}.  Hence, even if the sample has repeated samples, the guarantee from Theorem \ref{thm:short_hash_table} still holds.
\begin{lemma}\label{lemma:removing_duplicates}
For any $K^-,k^+\geq 0$ and $m>k^+ + k^-$:
\begin{align*}
    \inparen{\nsamples - (k^{+}+k^{-})}H\inparen{\frac{\nsamples-k^{+}}{\nsamples-(k^{+}+k^{-})}} \le \nsamples H\inparen{\frac{k}{\nsamples}}
\end{align*}
\end{lemma}
\begin{proof}
It is enough to show that for any two positive integers $a \le b$
\begin{align}
    b\cdot H\inparen{\frac{a}{b}} &\le (b+1) \cdot H\inparen{\frac{a}{b+1}}\label{eq:remove_duplicates_a} \\
    b\cdot H\inparen{\frac{a}{b}} &\le (b+1) \cdot H\inparen{\frac{a+1}{b+1}}\label{eq:remove_duplicates_b}
\end{align}
For Inequality \ref{eq:remove_duplicates_a}, we take the derivative of the function $f_1(a,b) \coloneqq b \cdot H(a/b)$ with respect to $b$ and show that it is always nonnegative. Indeed, the derivative of $f_1(a,b)$ with respect to $b$ is $\logv{b/(b-a)} > 0$. For Inequality \ref{eq:remove_duplicates_b}, we use $H(x) = H(1-x)$ and Inequality \ref{eq:remove_duplicates_a} to write
\begin{equation*}
    b\cdot H\inparen{\frac{a}{b}} = b \cdot H\inparen{\frac{b-a}{b}} \le (b+1)\cdot H\inparen{\frac{b-a}{b+1}} = (b+1) \cdot H\inparen{\frac{a+1}{b+1}} \qedhere
\end{equation*}
\end{proof}

%% file: refs.bib
@InProceedings{david_pacbayes,
author="McAllester, David",
% editor="Sch{\"o}lkopf, Bernhard
% and Warmuth, Manfred K.",
title="Simplified {PAC}-Bayesian Margin Bounds",
booktitle="Learning Theory and Kernel Machines",
year="2003",
publisher="Springer",
% address="Berlin, Heidelberg",
pages="203--215",
abstract="The theoretical understanding of support vector machines is largely based on margin bounds for linear classifiers with unit-norm weight vectors and unit-norm feature vectors. Unit-norm margin bounds have been proved previously using fat-shattering arguments and Rademacher complexity. Recently Langford and Shawe-Taylor proved a dimension-independent unit-norm margin bound using a relatively simple PAC-Bayesian argument. Unfortunately, the Langford-Shawe-Taylor bound is stated in a variational form making direct comparison to fat-shattering bounds difficult. This paper provides an explicit solution to the variational problem implicit in the Langford-Shawe-Taylor bound and shows that the PAC-Bayesian margin bounds are significantly tighter. Because a PAC-Bayesian bound is derived from a particular prior distribution over hypotheses, a PAC-Bayesian margin bound also seems to provide insight into the nature of the learning bias underlying the bound.",
% isbn="978-3-540-45167-9"
}

@book{cover2006elements,
  title={Elements of Information Theory},
  author={Cover, T.M. and Thomas, J.A.},
  % isbn={9780471241959},
  lccn={2005047799},
  series={A Wiley-Interscience publication},
  % url={https://books.google.it/books?id=j0DBDwAAQBAJ},
  year={2006},
  publisher={Wiley}
}

@article{bartlett2020benignpnas,
  author = {Peter L. Bartlett and Philip M. Long and G\'{a}bor Lugosi
and Alexander Tsigler},
  title = {Benign Overfitting in Linear Regression},
  journal = {Proceedings of the National Academy of Sciences},
  year = {2020},
  volume = {117},
  number = {48},
  pages = {30063--30070},
}

@article{montanari2020maxmarginasymptotics,
      title={The generalization error of max-margin linear classifiers: High-dimensional asymptotics in the overparametrized regime}, 
      author={Andrea Montanari and Feng Ruan and Youngtak Sohn and Jun Yan},
      year={2020},
      journal={Preprint, arXiv:1911.01544},
}

@article{hastie2020surprises,
%       title={Surprises in High-Dimensional Ridgeless Least Squares Interpolation}, 
%       author={Trevor Hastie and Andrea Montanari and Saharon Rosset and Ryan J. Tibshirani},
%       year={2020},
%       journal={Preprint, arXiv:1903.08560},
      
% }

@article{muthukumar2021classification,
  author = "Muthukumar, Vidya and Narang, Adhyyan and Subramanian, Vignesh and Belkin, Mikhail and Hsu, Daniel and Sahai, Anant",
  journal = "Journal of Machine Learning Research",
  title = "Classification vs regression in overparameterized regimes: Does the loss function matter?",
  number = "222",
  volume = "22",
  pages = "1-69",
  year = "2021"
}

@inproceedings{belkin2018overfittingperfectfitting,
 author = {Belkin, Mikhail and Hsu, Daniel J and Mitra, Partha},
 booktitle = {Advances in Neural Information Processing Systems (NeurIPS)},
 title = {Overfitting or perfect fitting? {R}isk bounds for classification and regression rules that interpolate},
 year = {2018}
}

@InProceedings{belkin2019interpolationoptimality,
  title = 	 {Does data interpolation contradict statistical optimality?},
  author =       {Belkin, Mikhail and Rakhlin, Alexander and Tsybakov, Alexandre B.},
  booktitle = 	 {International Conference on Artificial Intelligence and Statistics (AISTATS)},
  year = 	 {2019},
}

@article{chatterji2020linearnoise,
      title={Finite-sample analysis of interpolating linear classifiers in the overparameterized regime}, 
      author={Niladri S. Chatterji and Philip M. Long},
       journal = {Journal of Machine Learning Research},
  year    = {2021},
  volume  = {22},
  number  = {129},
  pages   = {1-30},
}

@inproceedings{negrea:in-defense,
    title={In Defense of Uniform Convergence: Generalization via derandomization with an application to interpolating predictors},
    author={Jeffrey Negrea and Gintare Karolina Dziugaite and Daniel M. Roy},
    year={2020},
    booktitle = {International Conference on Machine Learning},
    % eprint={1912.04265},
    % eprinttype={arXiv},
}

@article{nakkiran2020distributional,
  title={Distributional generalization: A new kind of generalization},
  author={Nakkiran, Preetum and Bansal, Yamini},
  journal={arXiv preprint arXiv:2009.08092},
  year={2020}
}

@inproceedings{
mallinar2022benign,
title={Benign, Tempered, or Catastrophic: Toward a Refined Taxonomy of Overfitting},
author={Neil Mallinar and James Simon and Amirhesam Abedsoltan and Parthe Pandit and Misha Belkin and Preetum Nakkiran},
booktitle={Advances in Neural Information Processing Systems},
% editor={Alice H. Oh and Alekh Agarwal and Danielle Belgrave and Kyunghyun Cho},
year={2022},
% url={https://openreview.net/forum?id=5oS20NUCJEX}
}

@article{koehler2021uniform,
  title={Uniform convergence of interpolators: Gaussian width, norm bounds and benign overfitting},
  author={Koehler, Frederic and Zhou, Lijia and Sutherland, Danica J and Srebro, Nathan},
  journal={Advances in Neural Information Processing Systems},
  volume={34},
  pages={20657--20668},
  year={2021}
}

@inproceedings{wang2022tight,
  title={Tight bounds for minimum $\ell_1$-norm interpolation of noisy data},
  author={Wang, Guillaume and Donhauser, Konstantin and Yang, Fanny},
  booktitle={International Conference on Artificial Intelligence and Statistics},
  pages={10572--10602},
  year={2022},
  organization={PMLR}
}

@article{misiakiewicz2022spectrum,
  title={Spectrum of inner-product kernel matrices in the polynomial regime and multiple descent phenomenon in kernel ridge regression},
  author={Misiakiewicz, Theodor},
  journal={arXiv preprint arXiv:2204.10425},
  year={2022}
}

@inproceedings{
xiao2022precise,
title={Precise Learning Curves and Higher-Order Scalings for Dot-product Kernel Regression  },
author={Lechao Xiao and Jeffrey Pennington and Theodor Misiakiewicz and Hong Hu and Yue Lu},
booktitle={Advances in Neural Information Processing Systems},
% editor={Alice H. Oh and Alekh Agarwal and Danielle Belgrave and Kyunghyun Cho},
year={2022},
% url={https://openreview.net/forum?id=HvJC_KsSx8S}
}

@inproceedings{raginsky2016information,
  title={Information-theoretic analysis of stability and bias of learning algorithms},
  author={Raginsky, Maxim and Rakhlin, Alexander and Tsao, Matthew and Wu, Yihong and Xu, Aolin},
  booktitle={2016 IEEE Information Theory Workshop (ITW)},
  pages={26--30},
  year={2016},
  organization={IEEE},
  natinote={sqrt bound}
}

@article{russo2019much,
  title={How much does your data exploration overfit? controlling bias via information usage},
  author={Russo, Daniel and Zou, James},
  journal={IEEE Transactions on Information Theory},
  volume={66},
  number={1},
  pages={302--323},
  year={2019},
  publisher={IEEE},
  natinote={sqrt bound}
}

@article{xu2017information,
  title={Information-theoretic analysis of generalization capability of learning algorithms},
  author={Xu, Aolin and Raginsky, Maxim},
  journal={Advances in Neural Information Processing Systems},
  volume={30},
  year={2017},
  natinote={sqrt bound}
}

@inproceedings{solomonoff1960preliminary,
  title={A preliminary report on a general theory of inductive inference},
  author={Solomonoff, Ray J},
  year={1960},
  % organization={Zator Company Cambridge, MA}
}

@article{valiant1984theory,
  title={A theory of the learnable},
  author={Valiant, Leslie G},
  journal={Communications of the ACM},
  volume={27},
  number={11},
  pages={1134--1142},
  year={1984},
  publisher={ACM New York, NY, USA}
}
